\documentclass[technote]{IEEEtran}

\usepackage[utf8]{inputenc}
\usepackage{graphicx}
\usepackage{amsmath,amssymb}
\usepackage{amsthm}
\usepackage{cite}
\usepackage{comment}

\theoremstyle{definition}
\newtheorem{definition}{Definition}
\theoremstyle{theorem}

\newtheorem{theorem}{Theorem}
\newtheorem{corollary}{Corollary}

\newcommand{\vc}[1]{\mathbf{#1}}

\begin{document}

\title{On the Stability and Generalization of Learning with Kernel Activation Functions}

\author{{Michele~Cirillo, Simone~Scardapane, Steven~Van~Vaerenbergh and~Aurelio~Uncini}%

\thanks{M. Cirillo, S. Scardapane and A. Uncini are with the Dept. of Information Engineering, Electronics and Telecommunications (DIET), ``Sapienza'' University of Rome, Via Eudossiana 18, 00184, Rome. S. Van Vaerenbergh is with the Dept. of Mathematics, Statistics and Computing, University of Cantabria, Avda. de los Castros 48, 39005 Santander, Cantabria, Spain. Corresponding author's email: simone.scardapane@uniroma1.it.}}%

\markboth{IEEE Transactions on Neural Networks and Learning Systems}%
{Cirillo \MakeLowercase{\textit{et al.}}: Stability of Learning with KAFs}

\maketitle

\begin{abstract}
In this brief we investigate the generalization properties of a recently-proposed class of non-parametric activation functions, the kernel activation functions (KAFs). KAFs introduce additional parameters in the learning process in order to adapt nonlinearities individually on a per-neuron basis, exploiting a cheap kernel expansion of every activation value. While this increase in flexibility has been shown to provide significant improvements in practice, a theoretical proof for its generalization capability has not been addressed yet in the literature. Here, we leverage recent literature on the stability properties of non-convex models trained via stochastic gradient descent (SGD). By indirectly proving two key smoothness properties of the models under consideration, we prove that neural networks endowed with KAFs generalize well when trained with SGD for a finite number of steps. Interestingly, our analysis provides a guideline for selecting one of the hyper-parameters of the model, the bandwidth of the scalar Gaussian kernel. A short experimental evaluation validates the proof.
\end{abstract}

\begin{IEEEkeywords}
Stability, generalization, kernel activation functions, stochastic gradient descent.
\end{IEEEkeywords}

\IEEEpeerreviewmaketitle

\section{Introduction}
\label{sec:int}

The recent successes of deep learning models (e.g., in machine translation) have provided a parallel boost in understanding theoretically their good generalization properties. In particular, several works have been devoted to the so-called `overfitting puzzle' \cite{zhang2017understanding,poggio2017theory}, i.e., the fact that highly over-parameterized neural networks (NNs), able to immediately memorize the entire training set \cite{zhang2017understanding} in principle, are nonetheless able to generalize well even with only moderate amounts of regularization. While this is counter-intuitive from the point of view of classical results in statistical learning theory (e.g., capacity measures such as the VC-dimension), a wide range of alternative explanations have been proposed to justify the strong empirical performance of NNs \cite{neyshabur2017exploring,kawaguchi2017generalization,raghu2017expressive,asadi2018chaining}.

The vast majority of these works has focused on a standard class of NNs, composed of linear projections (or convolutions) interleaved with fixed, element-wise nonlinearities, particularly the rectified linear unit (ReLU) \cite{du2018gradient}. Some have also explored the interplay of this type of NNs with the optimization hyper-parameters, e.g., the batch size \cite{hoffer2017train}, or with newer types of regularization, such as batch normalization. 

However, several authors have recently advocated the need for other types of architectures, especially by the use of flexible activation functions, able to learn per-neuron shapes during the training process. Examples of these are the maxout network \cite{goodfellow2013maxout}, the adaptive piecewise linear unit \cite{agostinelli2014learning}, and the kernel activation function (KAF) \cite{scardapane2019kafnets}, which is the main focus of this paper. While KAFs and similar models have shown good empirical performance compared to classical architectures (e.g., see also \cite{scardapane2018recurrent} for the recurrent case), they introduce significantly more flexibility into the architecture, and thus more potentiality for overfitting. For this reason, it is essential to supplement their evaluation with thorough analyses of their generalization properties, which in the case of KAFs have not been addressed yet by previous research.

\subsubsection*{Contributions of the paper} 
The aim of this paper is to analyze the generalization capabilities of NNs endowed with KAF nonlinearities. To this end, we exploit the analysis presented in \cite{hardt2016train}, whose roots can be found in previous works linking the generalization capabilities of a model to its \emph{algorithmic stability} \cite{bousquet2002stability}. More in detail, in \cite{hardt2016train} it was shown that a non-convex model trained via stochastic gradient descent (SGD) for a finite number of steps can generalize well, provided we are able to bound several constants related to its Lipschitz continuity and smoothness. Using this, we can obtain bounds on the generalization properties of KAFs by indirectly analyzing their smoothness and plugging back these results in the theorems from \cite{hardt2016train}. Interestingly, our main theorem in this sense (see Section \ref{sec:generalization}) provides a rigorous bound on one key hyper-parameter of the model, thus providing a practical guideline for using KAFs in real-world scenarios. We note that the general outline of our proof method is similar to the one introduced in \cite{eisenach2016nonparametrically}, from which we take a few results. However, because of the differences in the models we explore, the bulk of the proof differs significantly from \cite{eisenach2016nonparametrically}.

\subsubsection*{Organization of the paper}
The rest of this brief is organized as follows. In Section \ref{sec:preliminaries} we introduce the key concepts from \cite{hardt2016train} that will be used for our analysis. Section \ref{sec:kafnets} describes the KAF model from \cite{scardapane2019kafnets}, which is the focus of this brief. We prove the generalization of this model in Sections \ref{sec:generalization} and \ref{sec:proof_of_the_main_theorem}. After a small experimental evaluation in Section \ref{sec:experimental_evaluation}, we conclude in Section \ref{sec:conclusions} with insights on possible future work.

\section{Preliminaries}
\label{sec:preliminaries}

In this section, we recall some basic elements from the stability theory of \cite{hardt2016train}, on which we build our analysis. We denote by $f(\vc{w}; \vc{x})$ a generic NN, where we collect all adaptable parameters in $\vc{w}$, while $\vc{x} \in \mathbb{R}^m$ denotes an input vector (later on we will specialize our analysis to a specific architecture for $f$, based on KAFs \cite{scardapane2019kafnets}). Given a loss function $l(\cdot, \cdot)$ and an (unknown) probability distribution $p(\vc{x}, y)$ generating the data, we define the expected risk of $f$ under $l$ as:

\begin{equation}
    R[\vc{w}] \triangleq \mathbb{E}_{p} \, l\left( y, f(\vc{w}; \vc{x}) \right) \,.
    \label{eq:expected_risk}
\end{equation}

\noindent If we are only provided with a sample of $n$ i.i.d. draws from $p$ given by $S = \left\{\vc{x}_i, y_i\right\}_{i=1}^n$, the empirical risk is the finite sample approximation of \eqref{eq:expected_risk} using $S$:

\begin{equation}
    R_S[\vc{w}] \triangleq \frac{1}{n} \sum_{i=1}^n l(y_i, f(\vc{w}; \vc{x}_i )) \,.
    \label{eq:empirical_risk}
\end{equation}

\noindent We want to bound the excess risk of optimizing \eqref{eq:empirical_risk} in place of \eqref{eq:expected_risk}, when $\vc{w}$ is the result of some \textit{randomized learning algorithm} $\vc{w} = A(S)$. We have the following definitions.

\begin{definition}
The expected generalization error of a randomized algorithm $A$ is given by:
\[
\varepsilon_{\text{gen}} \triangleq \mathbb{E}_{S,A} \left[ R_S[A(S)] - R[A(S)] \right] \,,
\]
where expectation is taken both with respect to all possible training sets and runs of the algorithm.
\label{def:L}
\end{definition}

\begin{definition}
Algorithm $A$ is $\varepsilon$-uniformly stable if, for all data sets $S$, $S'$ differing for at most one example, we have:
\[
\sup_{(\vc{x}, y)} \mathbb{E}_A \left[ l(y, f(A(S); \vc{x})) - l(y, f(A(S'); \vc{x})) \right]  \le \varepsilon \,.
\]
\label{def:beta}
\end{definition}

\noindent Fundamentally for our analysis, stability of an algorithm implies generalization in expectation (see also \cite{bousquet2002stability}).

\begin{theorem}{(Hardt et al., 2016 \cite{hardt2016train})}
If $A$ is $\varepsilon$-uniformly stable, then $\lvert \varepsilon_{\text{gen}} \rvert \le \varepsilon$.
\end{theorem}

\noindent In particular \cite{hardt2016train} considers training with SGD, where at every iteration the estimate $\vc{w}_t$ is refined as:

\begin{equation}
    \vc{w}_{t+1} = \vc{w}_t - \mu_t \nabla l(y_{i(t)}, f(\vc{w}_t; \vc{x}_{i(t)})) \,,
\end{equation}

\noindent with the index $i(t)$ randomly chosen at every time step and $\mu_t$ is the step size. Before introducing the main theorem from \cite{hardt2016train}, we need two additional definitions.

\begin{definition}
A real-valued function $l$ is $L$-Lipschitz if for any points $\vc{w}'$ and $\vc{w}''$ in its domain we have:
\[
|l(\vc{w}')-l(\vc{w}'')|\leq \lVert \vc{w}' - \vc{w}'' \rVert \,.
\]
\end{definition}

\begin{definition}
A real-valued function $l$ is $\beta$-smooth if for any points $\vc{w}'$ and $\vc{w}''$ in its domain we have:
\[
\lVert \nabla l(\vc{w}') - \nabla l(\vc{w}'') \rVert \leq \beta \lVert \vc{w}' - \vc{w}'' \rVert \,.
\]
\end{definition}

\begin{theorem}{(\textit{Uniform stability of SGD})}
Assume that the loss function $l(\cdot, \cdot)$ is bounded, $L$-Lipschitz and $\beta$-smooth in terms of the vector $\vc{w}$, for any couple $(\vc{x},y)$. Assume we optimize it with SGD for a finite number $T$ of steps, using monotonically non-increasing step sizes $\mu_t \le c/t$, with $t=1,\ldots,T$ and $c$ a certain constant. Then, the algorithm is $\varepsilon$-uniformly stable with:
\begin{equation}
\varepsilon \le \frac{1+1/\beta c}{n - 1}\left(2cL^2\right)^{\frac{1}{\beta c + 1}}T^{\frac{\beta c}{\beta c + 1}} \,.
\label{eq:varepsilon}
\end{equation}
\end{theorem}

\noindent The proof can be found in Theorem 3.12 from \cite{hardt2016train}. Summarizing these results, for any architecture trained using SGD one can prove generalization bounds in this way: first, one proves that the loss function is bounded, \mbox{$L$-Lipschitz} and \mbox{$\beta$-smooth}; then, one uses the Theorems 1 and 2 to automatically infer uniform stability and, in turn, generalization bounds. 
Remarkably, this also provides an intuition which is somehow in contrast with the overfitting puzzle, i.e., it makes sense to use models for $f$ which are as expressive as possible, provided this additional flexibility translates to a faster rate of convergence, allowing to reduce the number of SGD steps taken during training. In the words of \cite{hardt2016train}: ``\textit{it may make sense for practitioners to focus on minimizing training time, for instance, by designing model architectures for which stochastic gradient method converges fastest to a desired error level}.''

\section{Kernel activation functions}
\label{sec:kafnets}

Consider now a feedforward, fully connected NN architecture:

\begin{equation}
    f(\vc{w};\vc{x}) = f_Q \circ f_{Q-1} \circ \ldots \circ f_1 \left(\vc{w}; \vc{x} \right) \,.
    \label{eq:nn}
\end{equation}

\noindent where each hidden layer $f_i$ consists of a certain number $H_i$ of functions, the neurons. More formally, let $f_{ij}(\vc{w};\cdot)$ denote the $j$th neuron of the $i$th layer. When applied to a $p\times1$ input vector $\vc{u}$, we have that $f_{ij}$ is of the type $f_{ij}=a_{ij} \circ g_{ij}$, where:
\begin{equation}
g_{ij}(\vc{w};\vc{u}) \triangleq \sum_{h=1}^p W_{ijh} \cdot u_h + b_{ij}\,,
\end{equation}
and $a_{ij}$ is a certain nonlinear transformation, the activation function. Some classical choices for such nonlinearities are $\tanh$ and ReLU, for which Lipschitz properties have received considerable attention \cite{wiatowski2018mathematical}. However, in this paper we consider a more expressive kind of activation functions, that can adapt their shape based on the training data, namely, that have parameters to be tuned during the learning process. In particular, for $i < Q$ we model each $a_{ij}$ as a KAF \cite{scardapane2019kafnets}, whose definition is given below, whereas, since we will consider multiclass classification problems, in the last layer, $i=Q$, we define the activation functions $a_{ij}$'s in order to constitute the softmax.
\begin{definition}
	A \textbf{kernel activation function} (KAF) \cite{scardapane2019kafnets} for layer $i$, node $j$ is defined as:
	\begin{equation}
	a_{ij}(s) = \sum_{k=1}^D \alpha_{ijk} \cdot \kappa \left(s, d_k \right) \,,
	\label{eq:kaf}
	\end{equation}
	where $s$ is a generic input to the KAF, $D$ is a hyper-parameter, $\alpha_{ijk}$ are the mixing coefficients, $\kappa$ is a one-dimensional kernel function, and  $d_1, \ldots, d_D$ are the dictionary elements.
\end{definition}
As in the original implementation of KAFs, we consider a separate set of mixing coefficients for every neuron (which are adapted via SGD together with the parameters of the linear projection), while the dictionary elements are fixed, shared across the entire network. In other words, in our scenario the parameter vector $\vc{w}$ is composed by the parameters $W_{ijh}$'s and $b_{ij}$'s, as usual, but also by the mixing coefficients $\alpha_{ijk}$'s of the activation functions in the hidden layers. The dictionary elements are defined by first choosing a value for $D$ and then sampling $D$ values across the $x$-axis uniformly with a certain sampling step. Generally speaking, increasing $D$ allows to increase the expressivity of the single activation function. 
Depending on the kernel function $\kappa$, one obtains different schemes for the functions' adaptation. In \cite{scardapane2019kafnets} and most of the later extensions, e.g. \cite{scardapane2018recurrent}, a Gaussian kernel is used:
\begin{equation}
\kappa(s, d_k) = e^{- \gamma (s - d_k)^2 }\,,
\label{eq:defe}
\end{equation}
\noindent where $\gamma$ is a second hyper-parameter (the inverse of the kernel bandwidth). Varying $\gamma$ has a clear intuitive meaning in varying the `receptive field' of each component of the linear expansion in \eqref{eq:kaf}. We consider this kernel for our analysis. A NN endowed with KAFs at all hidden layers is called a \textit{Kafnet}.

Wrapping up, by joining together the definitions~\mbox{(\ref{eq:nn})-(\ref{eq:defe})}, we have that the $j$th neuron in the $i$th inner layer, for $i=1,\ldots,Q-1$, is a function defined as:
\begin{equation}
a_{ij}(\vc{w};\vc{x}) \triangleq \sum_{k=1}^D \alpha_{ijk} \cdot e^{- \gamma (g_{ij}(\vc{w};\vc{x}) - d_k)^2}\,,
\label{eq:defa}
\end{equation}
with:
\begin{equation}
g_{ij}(\vc{w};\vc{x}) \triangleq \sum_{h=1}^{H_{i-1}} W_{ijh} \cdot a_{i-1h}(\vc{w};\vc{x})+b_{ij}.
\label{eq:defg}
\end{equation}
The function $g_{ij}$ represents the affine transformation applied on the inputs of the neuron. Such inputs, due to the feedforward and fully connected architecture of the considered network, are the outputs $a_{i-1j}$'s of the neurons from the previous layer. For the case $i=1$, we have that the quantity $a_{i-1j}$ is equal to the input $x_j$ and $H_{i-1}$ is equal to $m\triangleq\textnormal{card}\{\vc{x}\}$. The neurons in the output layer, where $i=Q$, are defined analogously, with the only difference that the softmax function is used in place of~\eqref{eq:defa}. 

\section{Generalization properties of Kafnets}
\label{sec:generalization}

To simplify reading, we state here our main result, and postpone its proof to the next section. As stated before, we focus on classification problems, although the method is easily extendable to regression tasks. For this, we assume the output function $f_Q$ in \eqref{eq:nn} is a softmax, while $l$ in \eqref{eq:empirical_risk} is the cross-entropy loss. As in the previous section, we use $i \in [1, \ldots, Q]$ to index a layer, $j \in [1, \ldots, H_i]$ to index a neuron inside a layer\footnote{Strictly speaking we should write $j_i$, but we drop the suffix for readability.}.

\begin{theorem}{(\textit{Smoothness of Kafnets})}
Let us assume that there exist constants $a$, $W$, $b$, $\alpha$ such that:
\begin{enumerate}
    \item For any $h$, $\lvert x_h \rvert \le a$;
    \item For any $i, j, h$, $\lvert W_{ijh} \rvert \leq W$ and $\lvert b_{ij} \rvert \le b$;
    \item For any $i, j, k$, $\lvert \alpha_{ijk} \rvert \le \alpha$.
\end{enumerate}

\noindent In addition, if we define $H \triangleq \max\left\{H_1, \ldots, H_Q, D\right\}$, under the following conditions:
\begin{enumerate}
    \setcounter{enumi}{3}
    \item $Q \ge 2$;
    \item The elements of the dictionary are sampled uniformly in $\left[-R, R\right]$ with $R \in \mathcal{O}(D)$;
    \item $\gamma H^2 \ge 1$,
\end{enumerate}

\noindent then the loss function in \eqref{eq:empirical_risk} is bounded, $L$-Lipschitz and $\beta$-smooth with respect to $\vc{w}$, with:
\begin{equation}
    L \in \mathcal{O}\left( \sqrt{Q} \left( \gamma H^4 \right)^{Q-1} \right) \,, \;\; \beta \in \mathcal{O}\left( Q\left(\gamma H^4\right)^{2(Q-1)} \right).
\end{equation}
We remark that the asymptotic notation used above is related to the variables $H$ and $\gamma$, whereas the parameters $a$, $W$, $b$ and $\alpha$ are intended to be constants of the problem, namely, belonging to $\mathcal{O}(1)$.
\label{theorem:kafnets_smoothness}
\end{theorem}
\begin{proof}
Due to the large number of computations involved, the proof is postponed to Section \ref{sec:proof_of_the_main_theorem}.
\end{proof}

\noindent Note that conditions (1)-(3) are similar to those used in \cite{eisenach2016nonparametrically} and can be enforced trivially. Condition (4) simply states that the network has at least one hidden layer. Conditions (5) and (6) are more interesting because they impose some constraints on the hyper-parameters of the model. In particular, condition (5) requires that the elements of the dictionary should be sampled from a sufficiently strict distribution, that increases at most linearly w.r.t. $H$, while condition (6) is a non-trivial constraint on the receptive field of each mixing coefficient.
Combining this result with the results from Section \ref{sec:preliminaries} gives us the desired property of generalization.

\section{Proof of the main theorem}
\label{sec:proof_of_the_main_theorem}

\subsection{Notational Conventions and Outline of the Proof.}
For the sake of clarity, in the following we will use a more compact notation to represent the quantities involved in~\eqref{eq:defa} and~\eqref{eq:defg}, introduced in Section~\ref{sec:kafnets}. In particular, we pose:
\begin{equation}
G_{ij} \triangleq g_{ij}(\vc{w};\vc{x}), 
~~~~~ A_{ij} \triangleq a_{ij}(\vc{w};\vc{x}),
~~~~~ E_{ijk} \triangleq e^{- \gamma (G_{ij} - d_k)^2 }.
\label{eq:defGAE}
\end{equation}
Furthermore, for any parameters $z$ and $w$ in $\vc{w}$, we define the generic partial derivatives of $G_{ij}$ as:
\begin{equation}
\frac{\partial G_{ij}}{\partial z} \triangleq G_{ij}'(z)\,, 
\quad 
\frac{\partial^2 G_{ij}}{\partial (z,w)} \triangleq G_{ij}''(z,w)\,. 
\end{equation}
Likewise, we define the partial derivatives of $E_{ijk}$ and $A_{ij}$ as $E_{ijk}'(z)$, $A_{ij}'(z)$, $E_{ijk}''(z,w)$ and $A_{ij}''(z,w)$.

Now we are ready to depict the outline of the proof, deferring some technical digressions in the next subsections. According to lemmas B.9 and B.10 in \cite{eisenach2016nonparametrically}, if the functions $G_{ij}'(z)$'s and $G_{ij}''(z,w)$'s are bounded by $Y_Q$ and $Z_Q$, respectively, then the loss function is bounded, $L$-Lipschitz and $\beta$-smooth. More precisely, in asymptotic notation the two lemmas state that:
\begin{equation}
L\in\mathcal{O}\left(\sqrt{Q} H Y_Q\right), \qquad \beta\in\mathcal{O}\left(Q H^2 \left(Y_Q^2+Z_Q\right)\right)\,.
\label{eq:Lbasym}
\end{equation} 
The core part of the proof consists of proving:
\begin{equation}
Y_Q\in\mathcal{O}\left( (\gamma H^3)(\gamma H^4)^{Q-2} \right), \quad Z_Q\in\mathcal{O}\left( (\gamma H^2)^2(\gamma H^4)^{2(Q-2)} \right)\,.
\label{eq:VWQasym}
\end{equation} 
In fact, once these bounds are computed, the statement easily follows by substituting~\eqref{eq:VWQasym} in~\eqref{eq:Lbasym} \footnote{The computations of these bounds heavily depend on the kind of activation functions that are employed. For this reason, even if our proof follows a similar outline of its counterpart in~\cite{eisenach2016nonparametrically}, it is heavily different in its technical apparatus, due to the very different characteristics of the activation functions involved here.}. Therefore, in the reminder of the proof we will focus on proving~\eqref{eq:VWQasym}.

\subsection{Some intermediate results.}
In this subsection we collect some intermediate results that are needed to prove~\eqref{eq:VWQasym}. Before continuing, we present an useful inequality that will be repeatedly used in the next, $|a+b|\leq|a|+|b|$, that is well known as triangle inequality.

We start with three boundary inequalities for the functions $E_{ijk}$'s, $A_{ij}$'s and $G_{ij}$'s. First, as a direct consequence of definition~\eqref{eq:defe}, and using the terminology introduced in~\eqref{eq:defGAE}, for any $i,j,k$ we have:
\begin{equation}
	|E_{ijk}| \leq 1 \,.
	\label{eq:boundsE}
\end{equation}
Likewise, according to~\eqref{eq:defa} and~\eqref{eq:defGAE}, we have that \mbox{$A_{ij}=\sum_{k=1}^D \alpha_{ijk} \cdot E_{ijk}$}. For the triangle inequality, we have \mbox{$|A_{ij}| \leq \sum_{k=1}^D |\alpha_{ijk}| \cdot |E_{ijk}|$}, therefore, using~\eqref{eq:boundsE} and the hypothesis $|\alpha_{ijk}| \leq \alpha$, for any $i,j$ we get:
\begin{equation}
|A_{ij}| \leq D \alpha 
\label{eq:boundsA}
\end{equation}
Finally, from~\eqref{eq:defa} and~\eqref{eq:defGAE}, we have \mbox{$G_{ij}=\sum_{k=1}^{H_{i-1}} W_{ijh} \cdot A_{i-1h} + b_{ij}$}.
Again, for the triangle inequality, we have \mbox{$|G_{j}| \leq \sum_{h=1}^{H_{i-1}} |W_{ijh}| \cdot |A_{i-1h}| + |b_{ij}|$}. Therefore, when $i=1$, recalling that $A_{0h}=x_h$ and $H_0=m$, and using the hypotheses $|x_h| \leq a$, $|W_{ijh}| \leq W$ and $|b_{ij}|\leq b$, we have: 
\begin{equation}
|G_{1j}| \leq m W a + b
\label{eq:boundsG1}
\end{equation}
Contrariwise, when $i\geq 2$, the hypotheses $|W_{ijh}| \leq W$ and $|b_{ij}|\leq b$, together with the previous result in~\eqref{eq:boundsA}, lead to:
\begin{equation}
|G_{ij}| \leq H_{i-1} W D \alpha + b
\label{eq:boundsG}
\end{equation}

Now, we analyse the bounds of the first and second order derivatives of the functions $E_{ijk}$'s, $A_{ij}$'s and $G_{ij}$'s. In particular, these results hold under the assumption that, given a fixed $i$, for any $j$:
\begin{equation}
\left|G_{ij}\right| \leq X_i\,, \quad \left|G_{ij}'(z)\right| \leq Y_i\,, \quad \left|G_{ij}''(z,w)\right| \leq Z_i\,.
\label{eq:hpUVW}
\end{equation} 
From the definition of $E_{ijk}$ in~\eqref{eq:defe}, we obtain:
\begin{eqnarray}
E_{ijk}'(z)=-2 \gamma E_{ijk} \left( G_{ij} - d_k \right) G_{ij}'(z),~~~~~~~~~~~~~~~~~~~~~~~~~~~~~~~\,
\nonumber\\
E_{ijk}''(z,w)=-2 \gamma E_{ijk} [ ( 1-2\gamma ( G_{ij} - d_k )^2 ) G_{ij}'(z) G_{ij}'(w) ~~~~~~~~~~~~~~
\nonumber\\+~ (G_{ij} - d_k) G_{ij}''(z,w)],~~~~~~~~~~~~~~~~~~~~~~~~~~~~~~~~~~~~~\,
\nonumber
\end{eqnarray}
Since, for the triangle inequality, \mbox{$|G_{ij} - d_k| \leq |G_{ij}| + |d_k|$}, and since for hypothesis $|d_k| \leq R$ and $|G_{ij}|\leq X_i$, and recalling also that $|E_{ijk}| \leq 1$, we obtain:
\begin{eqnarray}
|E_{ijk}'(z)|&\leq& 2\gamma(X_i + R)Y_i \,,
\nonumber\\
|E_{ijk}''(z,w)|&\leq& 2\gamma[(1+2\gamma(X_i + R)^2)Y_i^2+(X_i + R)Z_i]\,.
\nonumber\\
\label{eq:boundsEppp}
\end{eqnarray}
Now, we focus on the functions $A_{ij}'(z)$'s and $A_{ij}''(z,w)$'s. Let us define \mbox{$\alpha_{ij}\triangleq \{\alpha_{ijk}\}_{k=1}^D$}, from the definition of $A_{ij}$ in~\eqref{eq:defa} we have:
\begin{eqnarray}
& A_{ij}'(z) = \begin{cases}
E_{ijk}, & \mbox{if $z\in\alpha_{ij}$} 
\\
\\
\displaystyle{\sum_{k=1}^D} \alpha_{ijk} \cdot E_{ijk}'(z), & \mbox{if $z\notin\alpha_{ij}$}
\end{cases}
\nonumber
\end{eqnarray}
and, recalling that $A_{ij}''(z,w)=\partial A_{ij}'(z) / \partial w$, we have: 
\begin{eqnarray}
& A_{ij}''(z,w) = \begin{cases}
E_{ijk}'(w), & \mbox{if $z\in\alpha_{ij}$, $w\notin\alpha_{ij}$} 
\\
\\
\displaystyle{\sum_{k=1}^D} \alpha_{ijk} \cdot E_{ijk}''(z,w), & \mbox{if $z\notin\alpha_{ij}, w\notin\alpha_{ij}$}
\\
\\
0, & \mbox{otherwise}
\end{cases}
\nonumber
\end{eqnarray} 
As usual, since we have $|E_{ijk}|\leq 1$, and, for the triangle inequality, \mbox{$|\sum_{k=1}^D \alpha_{ijk} \cdot E_{ijk}'(z)| \leq \sum_{k=1}^D |\alpha_{ijk}| \cdot |E_{ijk}'(z)|$}, by using the hypothesis $|\alpha_{ijk}| \leq \alpha$ and~\eqref{eq:boundsEppp} we obtain:
\begin{equation}
|A_{ij}'(z)| \leq \max\left\{1,\, 2 D \alpha \gamma(X_i + R)Y_i \right\}
\label{eq:boundsAp}
\end{equation}
With analogous considerations as above, we obtain:
\begin{eqnarray}
& & |A_{ij}''(z,w)| \nonumber\\
& \leq & \max\left\{\begin{array}{c}
2 \gamma(X_i + R)Y_i, \\
2 D \alpha \gamma [(1+2\gamma(X_i + R)^2)Y_i^2+(X_i + R)Z_i] 
\end{array}\right\}\,. \nonumber\\
\label{eq:boundsApp}
\end{eqnarray}
Finally, we prove two boundary inequalities for the derivatives of the functions $G_{ij}$'s. Let us define \mbox{$W_{ij}\triangleq \{W_{ijh}\}_{h=1}^{H_{i-1}}$}. According to the definition of $G_{ij}$ in~\eqref{eq:defg}, we have:
\begin{eqnarray}
& G_{ij}'(z) = \begin{cases}
A_{i-1h}, & \mbox{if $z\in W_{ij}$} 
\\
\\
1, & \mbox{if $z=b_{ij}$}
\\
\\
\displaystyle{\sum_{h=1}^{H_{i-1}}} W_{ijh} \cdot A_{i-1h}'(z), & \mbox{if $z\notin W_{ij} \cup \{b_{ij}\}$}
\end{cases}
\nonumber
\end{eqnarray}
and, since $G_{ij}''(z,w)=\partial G_{ij}'(z) / \partial w$, we have: 
\begin{eqnarray}
& G_{ij}''(z,w) = \begin{cases}
A_{i-1h}'(w), 
& \mbox{if } \begin{cases}
z\in W_{ij} 
\\ w\notin W_{ij} \cup \{b_{ij}\}
\end{cases}
\\
\\
\displaystyle{\sum_{h=1}^{H_{i-1}}} W_{ijh} \cdot A_{i-1h}''(z,w), & \mbox{if } \begin{cases}
z\notin W_{ij} \cup \{b_{ij}\} \\ w\notin W_{ij} \cup \{b_{ij}\}
\end{cases} 
\\
\\
0, & \mbox{otherwise}
\end{cases}
\nonumber
\end{eqnarray}
When $i=1$, recalling that $A_{i-1h}$ is constant and equal to $x_h$, we simply have $A_{i-1h}'(z)=0$ and $A_{i-1h}''(z,w)=0$ for any $z,w\in\vc{w}$, and we also have $|A_{i-1h}|=|x_h|\leq a$. Therefore:
\begin{equation}
|G_{1,j}'(z)| \leq \max \left\{1,a\right\}\,,
\label{eq:boundsG1p}
\end{equation}
\begin{equation}
|G_{1,j}''(z,w)| \leq 0\,.
\label{eq:boundsG1pp}
\end{equation}
Otherwise, when $i\geq 2$, using the triangle inequality and the boundary results~\eqref{eq:boundsA}, \eqref{eq:boundsAp} and~\eqref{eq:boundsApp}, we get:
\begin{equation}
|G_{ij}'(z)| \leq \max\left\{
\begin{array}{c}
D \alpha, 
\\1, 
\\2 H_{i-1} W D \alpha \gamma(X_{i-1} + R)Y_{i-1},
\end{array}
\right\}\,.
\label{eq:boundsGp}
\end{equation}
and $|G_{ij}''(z,w)|$ is less or equal than the maximum among:
\begin{equation}
\begin{array}{c}
1,\, \\ 
2 D \alpha \gamma(X_{i-1} + R)Y_{i-1}, \\
2 H_{i-1} W \gamma(X_{i-1} + R)Y_{i-1}, \\
2 H_{i-1} W D \alpha \gamma [(1+2\gamma(X_{i-1} + R)^2)Y_{i-1}^2+(X_{i-1} + R)Z_{i-1}] 
\end{array} \\
\label{eq:boundsGpp}
\end{equation}
In particular, for obtaining these results we used the identities \mbox{$\max\{A,A,B\}=\max\{A,B\}$}, \mbox{$\max\{A,\max\{B,C\}\}=\max\{A,B,C\}$} and $c\cdot\max\{A,B\}=\max\{c \cdot A,c \cdot B\}$, for any $c\geq0$.

\subsection{Asymptotic bounds of the affine transformations.}
In this subsection we conclude the proof, proving the missing piece~\eqref{eq:VWQasym}. To this aim, by exploiting the recursive structure of the network, we prove by induction a more general result, such that, for any fixed $i \geq 2$, the functions $G_{ij}'(z)$'s and $G_{ij}''(z,w)$'s are all bounded by: 
\begin{equation}
Y_i\in\mathcal{O}\left( (\gamma H^3)(\gamma H^4)^{i-2} \right), \quad Z_i\in\mathcal{O}\left( (\gamma H^2)^2(\gamma H^4)^{2(i-2)} \right)\,.
\label{eq:VWasym}
\end{equation}
Then, since we assumed $Q\geq2$, this general result implies the desired one in~\eqref{eq:VWQasym}. Preliminarily, we recall that the asymptotic notation used here is intended to be relative to the variables $H_i$'s and $\gamma$, whereas the parameters $m$, $a$, $W$, $b$, $\alpha$ are intended to be constants. Technically speaking, this means that, for example, expressions like $2 D \alpha \gamma H_i + b$ belong to $\mathcal{O}(\gamma H_i)$. In the next, we will repeatedly use some well known properties of the asymptotic notation as, for example, if $x\in\mathcal{O}(f)$ and $y\in\mathcal{O}(g)$ then $x+y\in\mathcal{O}(\max\{f,g\})$, $x\cdot y\in\mathcal{O}(f \cdot g)$ and $\max\{x,y\}\in\mathcal{O}(\max\{f,g\})$.

We start by inductively proving the first part of~\eqref{eq:VWasym} about the $Y_i$'s. Let us consider the base case $i=2$. From~\eqref{eq:boundsG1} and~\eqref{eq:boundsG1p}, we note that the bounds of the functions $G_{1j}$'s and $G_{1j}'(z)$'s do depend neither by $H_2$ nor by $\gamma$, and, therefore, they are two constants $X_1\in\mathcal{O}(1)$ and $Y_1\in\mathcal{O}(1)$.
Furthermore, using the hypothesis $R\in\mathcal{O}(D)$, and recalling that $H \triangleq \max\left\{H_1, \ldots, H_Q, D\right\}$, we have that \mbox{$(X_1+R)\in\mathcal{O}(\max\{1,H\})=\mathcal{O}(H)$}, and, therefore:
\begin{equation}
\begin{array}{c}
D \alpha \in\mathcal{O}(H),\\
1 \in\mathcal{O}(1), \\
2 H_{1} W D \alpha \gamma \cdot (X_{1} + R)Y_{1} \in\mathcal{O}( \gamma H^2 \cdot H ),
\end{array}
\label{eq:asym1}
\end{equation}
Since we know $X_1$ and $Y_1$, we can use~\eqref{eq:boundsGp}, and by using~\eqref{eq:asym1} we can rewrite it in asymptotic notation. In this way, we get that the boundary $Y_2$ of the functions $G_{2j}'(z)$'s is in the order of:
\begin{equation}
\mathcal{O}(\max\{H\,, 1\,, \gamma H^3 \})=\mathcal{O}(\gamma H^3)\,,
\label{eq:boundsG2p_asym}
\end{equation}
where the last equality follows from the assumption $\gamma H^2 \geq 1$ and from the obvious inequality $H\geq 1$. Since the result~\eqref{eq:boundsG2p_asym} matches the claim in~\eqref{eq:VWasym} for the case $i=2$, we proved the base case. 

Now, we turn on the inductive step, namely, we prove~\eqref{eq:VWasym} assuming that it holds for $i-1$. First, from~\eqref{eq:boundsG} we note that when $i>2$ the functions $G_{i-1j}$'s have bound $X_{i-1}$ such that $X_{i-1}\in\mathcal{O}(H^2)$.
We have again $D \alpha \in\mathcal{O}(H)$ and $1 \in\mathcal{O}(1)$, like for the base case, but now, since \mbox{$(X_{i-1}+R)\in\mathcal{O}(\max\{H^2,H\})=\mathcal{O}(H^2)$}, and since $Y_{i-1}$ fulfils~\eqref{eq:VWasym}, we have:
\begin{eqnarray}
& & 2 H_{i-1} W D \alpha \gamma \cdot (X_{i-1} + R)Y_{i-1} \nonumber\\
& \in & \mathcal{O}\left(\gamma H^2 \cdot H^2(\gamma H^3)(\gamma H^4)^{i-2} \right) \nonumber\\
& = & \mathcal{O}\left((\gamma H^3)(\gamma H^4)^{i-1} \right).
\label{eq:asym2}
\end{eqnarray}
By using~\eqref{eq:asym2}, we can rewrite~\eqref{eq:boundsGp} in asymptotic notation for the case $i>2$. We obtain that the boundary $Y_{i}$ of functions $G_{ij}'(z)$'s is in the order of:
\begin{equation}
\mathcal{O}(\max\{H\,, 1\,, \gamma H^3 (\gamma H^4)^{i-1} \})=\mathcal{O}(\gamma H^3 (\gamma H^4)^{i-1}).
\label{eq:boundsGip_asym}
\end{equation}
For the boundaries $Z_i$'s, the same identical arguments apply. For space constraints, we provide only a brief sketch of their proof that is, anyway, complete of all the most technical intermediate results. For the base case $i=2$, since $G_{1j}''(z,w)=0$ and $(X_{i-1}+R)\in\mathcal{O}(H)$, the asymptotic counterpart of~\eqref{eq:boundsGpp} becomes:
\begin{equation}
\mathcal{O}\left(\max\{1,\gamma H^2,\gamma H^2,(\gamma H^2)^2\}\right)=\mathcal{O}\left((\gamma H^2)^2\right)\,.
\end{equation}
Again, the last equality holds because from hypothesis \mbox{$\gamma H^2 \geq 1$}. Finally, assuming that the second part of~\eqref{eq:VWasym} holds for $i-1$, the asymptotic version of~\eqref{eq:boundsGpp} is:
\begin{equation}
\mathcal{O}\left(\max\left\{ 
\begin{array}{c}
	1,\\
	\gamma H^3 (\gamma H^3)(\gamma H^4)^{i-2},\\
	\gamma H^3 (\gamma H^3)(\gamma H^4)^{i-2},\\
	\gamma H^2 (\gamma H^4) (\gamma H^3)^2(\gamma H^4)^{2(i-2)},\\
	\gamma H^2(H^2)(\gamma H^2)^2(\gamma H^4)^{2(i-3)}
\end{array} \right\}\right)
\end{equation}
The maximum returns the fourth entry (recall that \mbox{$\gamma H^2 \geq 1$}), that in compact form reduces to $(\gamma H^2)^2(\gamma H^4)^{2(i-1)}$, namely, to the claim~\eqref{eq:VWasym}.

\section{Experimental evaluation}
\label{sec:experimental_evaluation}

\begin{figure}
    \centering
    \includegraphics[width=0.9\columnwidth]{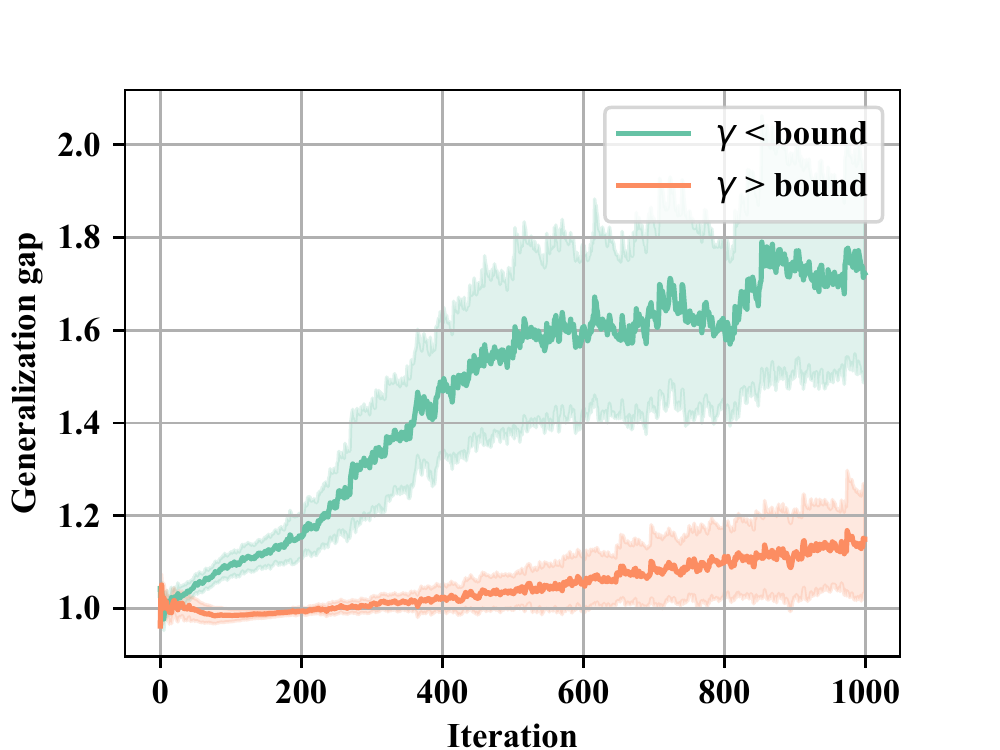}
    \caption{Generalization gap for two models trained with different $\gamma$ hyper-parameters.}
    \label{fig:generalization_gap}
\end{figure}

In this section we provide a small experimental verification of our main theorem (Theorem \ref{theorem:kafnets_smoothness}). We note that this section is not intended to verificate experimentally KAFs, as this was done extensively in the original publication \cite{scardapane2019kafnets} and later works. Instead, we showcase a simplified toy scenario to test the validity of our bounds.

We use the experimental procedure from \cite{guyon2003design} (which is also implemented in the scikit-learn Python library) to generate a simple two-class classification problem, where each class is described by two clusters randomly assigned to the vertices of a $2$-dimensional hypercube. To this we add two additional random features with no correlation to the actual class. We sample $1000$ points as our training set, and another $1000$ for testing. We compare two NNs with one $10$-dimensional hidden layer each endowed with KAFs nonlinearities. In both cases, each KAF has $20$ elements in the dictionary sampled uniformly from $\left[-3.0, 3.0\right]$, while the mixing coefficients are randomly initialized from a normal distribution. For the first KAF we select $\gamma=1.0$, while for the second one we select a smaller $\gamma=0.005$. We remark that only the former choice is consistent with the condition $\gamma H^2 \ge 1$ in Theorem 3. Furthermore, we remark that the choice $\gamma=1$ and the values assigned to the other hyperparameters are coherent with the ones considered in \cite{scardapane2019kafnets}); this is an important point, since we are interested to test generalization bounds for a Kafnet that already proved to work well in practice. We train the two networks with the Adam optimization algorithm on random mini-batches of $32$ elements, evaluating at every step the empirical risk on a separate mini-batch taken from the test portion of the dataset.

Both models converge to a final training loss which is $\approx 10^{-1}$ in only a few iterations. However, we plot in Fig. \ref{fig:generalization_gap} the generalization gap of the two models, that we compute as the ratio between the test empirical risk and the training empirical risk. As can be seen, the model satisfying the assumptions of our theorem only starts to overfit after a few hundred iterations, while the model with the lower $\gamma$ starts to overfit almost immediately, with a final empirical risk which is almost double for the new data than for the training data.

\section{Conclusions and discussion}
\label{sec:conclusions}

In this paper we provided an analysis of the generalization capabilities of a new class of non-parametric activation functions, i.e., kernel activation functions (KAFs). The analysis builds on top of recent results with respect to the interplay between stochastic gradient descent (SGD) and stability of the functions being optimized. Our main theorem shows that a NN endowed with KAFs is stable and generalizes well when trained with SGD for a finite number of steps, provided it satisfies a simple bound on one of its hyper-parameters.

Here, we focused on a specific variant of KAF, which is the one used in \cite{scardapane2019kafnets}. For future research, we aim at generalizing our theorems to more types of kernel functions. In addition, we are eager to investigate different types of generalization analyses. For example, in \cite{marra2018learning} it is shown that an activation function similar to KAFs can appear from an optimization problem formulated in a functional space. As a large literature has been devoted to complexity measures for these scenarios, it would be interesting to explore the interplay and possible extensions of these tools to the case considered here.

\bibliographystyle{IEEEtran}
\bibliography{biblio}

\end{document}